\documentclass{article}

\usepackage{microtype}
\usepackage{graphicx}
\usepackage{subcaption}
\usepackage{booktabs} 
\usepackage{enumitem}
\usepackage{commath}
\usepackage{amssymb}
\usepackage{graphicx} 
\usepackage[skip=2pt]{caption}
\usepackage{xspace}
\usepackage{xcolor}
\usepackage{thmtools}
\usepackage{amsthm}

\usepackage{hyperref}

\usepackage[accepted]{icml2021}

\icmltitlerunning{How to Not Measure Disentanglement}

\newcommand{\OurMeasure}{3CharM}
\providecommand{\argmax}{\operatornamewithlimits{arg\,max}}
\newtheorem{property}{Property}
\newtheorem{fact}{Fact}
\newtheorem{characteristic}{Characteristic}
\newtheorem{definition}{Definition}

\begin{document}

\twocolumn[
\icmltitle{How to Not Measure Disentanglement}



\icmlsetsymbol{equal}{*}

\begin{icmlauthorlist}
\icmlauthor{Anna Sepliarskaia}{tuWien}
\icmlauthor{Julia Kiseleva}{mrs}
\icmlauthor{Maarten de Rijke}{uva}
\end{icmlauthorlist}

\icmlaffiliation{tuWien}{Vienna University of Technology, Austria}
\icmlaffiliation{uva}{University of Amsterdam, Amsterdam, The Netherlands}
\icmlaffiliation{mrs}{Microsoft Research AI, Seattle, WA, USA}

\icmlcorrespondingauthor{Anna Sepliarskaia}{a.sepliarskaia@uva.nl}
\icmlcorrespondingauthor{Julia Kiseleva}{Julia.kiseleva@microsoft.com}
\icmlcorrespondingauthor{Maarten de Rijke}{derijke@uva.nl}

\icmlkeywords{representation learning, disentangled representation, evaluation}

\vskip 0.3in
]



\printAffiliationsAndNotice{\icmlEqualContribution} 

\begin{abstract}
To evaluate disentangled representations several metrics have been proposed. 
However, theoretical guarantees for conventional metrics of disentanglement are missing. 
Moreover, conventional metrics do not have a consistent correlation with the outcomes of qualitative studies. 
In this paper we analyze metrics of disentanglement and their properties. 
We conclude that existing metrics of disentanglement were created to reflect different characteristics of disentanglement and do not satisfy two basic desirable properties: (1) assign a high score to representations that are disentangled according to the definition; and (2) assign a low score to representations that are entangled according to the definition. 
In addition, we propose a new metric of disentanglement and prove that it satisfies both of the properties.
\end{abstract}

\section{Introduction}
Algorithms for learning representations are crucial for a variety of machine learning tasks, including image classification~\citep{vincent2008extracting, hinton2006reducing} and image generation~\citep{goodfellow2014generative, makhzani2015adversarial}.
One type of representation learning algorithm is designed to create a disentangled representation.
While there is no standardized definition of a disentangled representation, the key intuition is that a disentangled representation should capture and separate the generative factors ~\citep{bengio2013representation, higgins2018towards}.
In this paper, we assume that the \emph{generative factors} of the dataset are interpretable factors that describe every sample from the dataset.

Consider, for example, a dataset containing rectangles of different shapes.
There are two generative factors for this dataset: the length and width of the rectangles (see Fig.~\ref{ex}).
In the disentangled latent representation of this dataset we can choose two latent factors.
One of these factors is an invertible function of the length of the rectangles.
Another is an invertible function of the width of the rectangles.
\begin{figure}

\centering
\includegraphics[clip,trim=0mm 200mm 0mm 200mm,width=0.5\textwidth]{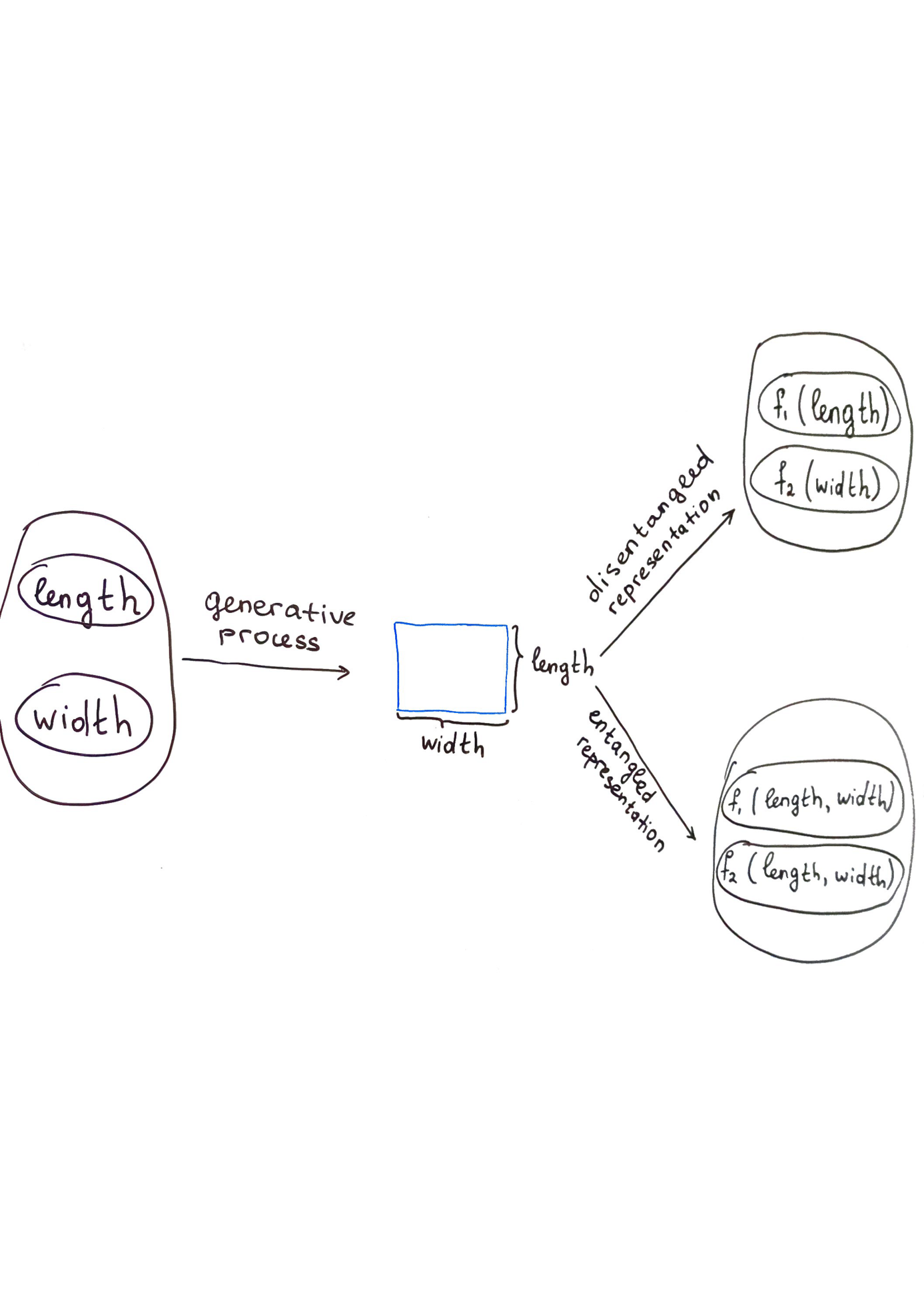}
\caption{Example of a dataset containing rectangles}
\label{ex}
\end{figure}

Learning a disentangled representation is an important step towards better representation learning because a disentangled representation contains information about elements in a dataset in an interpretable and compact structure~\citep{bengio2013representation, higgins2018towards}.
Therefore, the development of an algorithm that learns disentangled representations has become an active area of research~\cite{detlefsen2019explicit, dezfouli2019disentangled,lorenz2019unsupervised}.
The conventional way to measure the quality of these algorithms is to provide the results according to one of following metrics~\citep{locatello2018challenging}: \emph{BetaVAE}~\citep{higgins2016beta}, FactorVAE~\citep{kim2018disentangling}, DCI~\citep{eastwood2018framework}, SAP score~\citep{kumar2017variational}, and MIG~\citep{chen2018isolating}.
However, it was shown that the outcomes of these metrics are inconsistent with the outcomes of a qualitative study of the disentanglement of learned representations~\citep{abdi2019preliminary}; moreover, it is not clear which metric should be preferred.

In this paper, we theoretically analyzed the conventional metrics.
The outcome of our analysis is understanding of the reasons why conventional metrics not always correlate with each other: different metrics were designed to reflect different characteristics of disentanglement.
As a consequence, the metric for evaluating the algorithm should be determined by which of the characteristic of disentanglement the method was designed to reflect.
Moreover, we understood why outcomes of conventional metrics are inconsistent with definition of disentanglement.
We checked if the metrics satisfy basic desirable properties: (1) assign a high score to representations that are disentangled according to the definition; and (2) assign a low score to representations that are entangled according to the definition. 
We showed that majority of the metrics do not satisfy them.

To conclude, our key contributions in this paper are:
\begin{itemize}[leftmargin=*,nosep]
\item We review existing metrics of disentanglement and discuss their fundamental properties.
\item We propose a new metric of disentanglement with theoretical guarantees,
and establish its fundamental properties.
\end{itemize}

\section{Background and Notation}
\subsection{Representation learning}
There are different types of representation learning algorithm, but usually, an algorithm for learning disentangled representations consists of two parts: an encoder and a decoder.
An \emph{encoder} is a function:
\begin{equation}
f_e : \mathbb{R}^d \rightarrow \mathbb{R}^N,\ \bold{c}=f_e(\bold{x}), 
\end{equation}
where $\bold{c}$ is a latent representation of the data sample $\bold{x}$.

A \emph{decoder} is a function: 
\begin{equation}
f_d : \mathbb{R}^N \rightarrow \mathbb{R}^d,\ f_d(f_e(\bold{x}))\sim \bold{x}, 
\end{equation}
where $f_d(f_e(\bold{x}))$ should be close to $\bold{x}$. 

\subsection{Ground truth generative factors}
We assume that a dataset was generated by a generative process from \emph{generative factors}.
We define the \emph{generative factors} of a dataset in the following way:
\begin{definition}
\label{gen_fac_defin_first}
\rm
The \emph{generative factors} of a dataset are interpretable factors that describe the difference between any two samples from $X$.
\end{definition}

Although there are different ways of choosing the generative factors, in this paper we assume that the meaning of the generative factors is given to us and fixed.
Moreover, we are interested in whether the representation is disentangled \emph{corresponding} to generative factors given to us.
For example, the generative factors for a dataset containing rectangles of different shapes, as shown in Fig. ~ \ref{ex}, are the length and width of the rectangles.
Although, the length of the diagonal and the width of the rectangle are another set of generative factors of the rectangles, we will assume that we are interested in a disentangled representation \emph{corresponding} to the length and width of the rectangles.

We further define the \emph{ground truth} generative factors:
\begin{definition}
\label{gt_gen_fac_defin_first}
\rm
The \emph{ground truth} generative factors are the values of generative factors for a given collection.
\end{definition}
This means that we assume that for each sample $\bold{x}\in X$ of the dataset, the values of the generative factors $\mathbf{z}\in\mathbb{R}^{K}$ are known during the evaluation.

\section{Metrics of Disentanglement of Representations}
The main purpose of this paper  is to analyze conventional metrics of disentangled representations, which is done in this section.
Though there is no universally accepted definition of disentanglement, most metrics are based on the definition proposed in \citep{bengio2013representation} and reflect characteristics of a disentangled representation in accordance with this definition.
However, conventional metrics were designed to reflect \textbf{different} characteristics of disentangled representations: conventional metrics can be divided into two groups, depending on which characteristic they reflect.
In this paper, we analyze whether conventional metrics satisfy the following fundamental properties:
\begin{property}
\label{property_1}
A metric gives a high score to all representations that satisfy the characteristic that the metric reflects.
\end{property}
\begin{property}
\label{property_2}
A metric gives a low score for all representations that do not satisfy the characteristic that the metric reflects.
\end{property}
\subsection{BetaVAE, FactorVAE and DCI}

\begin{figure}

\begin{subfigure}[t]{0.48\linewidth}
\includegraphics[width=\linewidth]{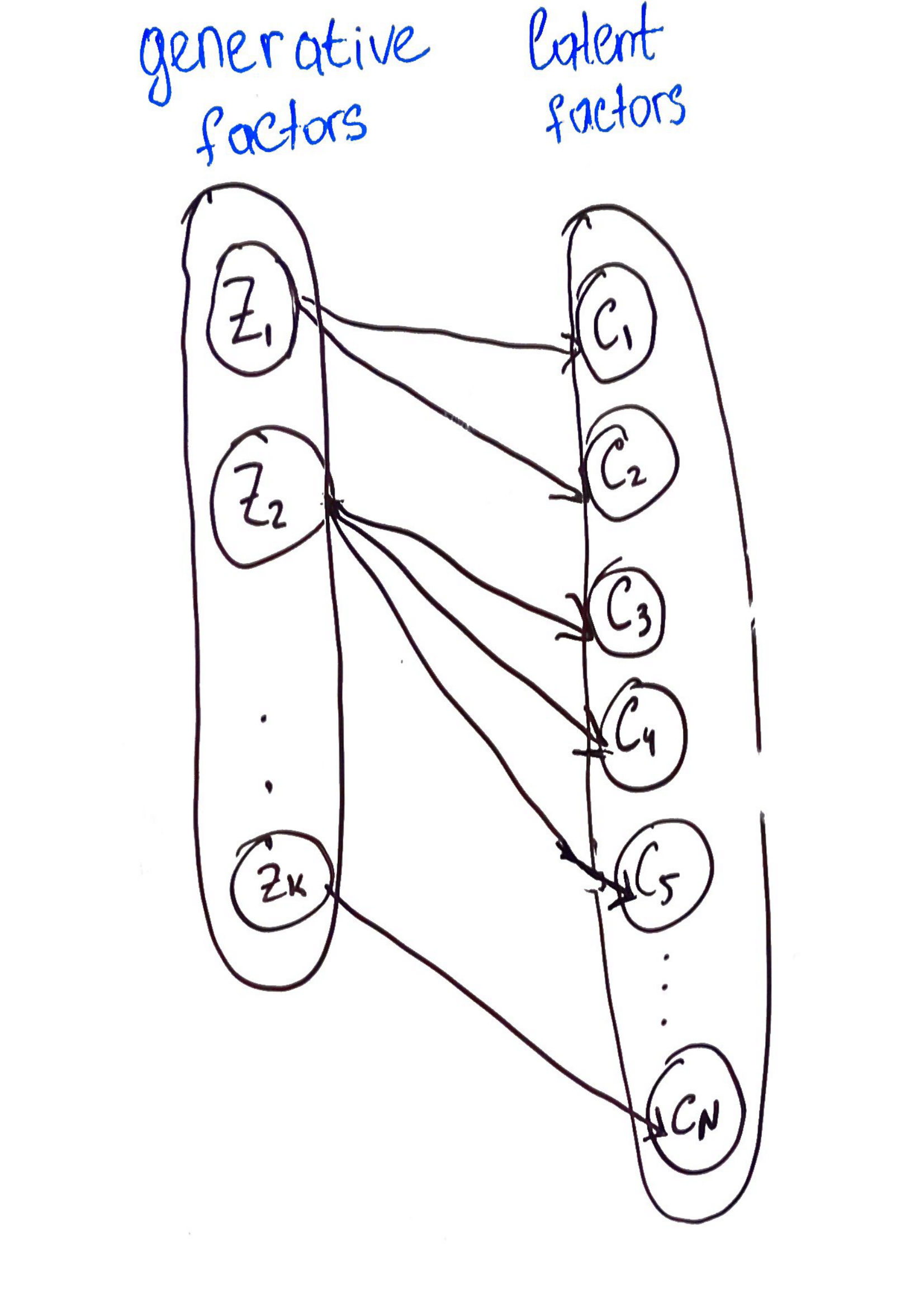}
\caption{First characteristic of disentanglement}
\label{First_caharact}
\end{subfigure}%
\hfill
\begin{subfigure}[t]{0.48\linewidth}
\includegraphics[width=\linewidth]{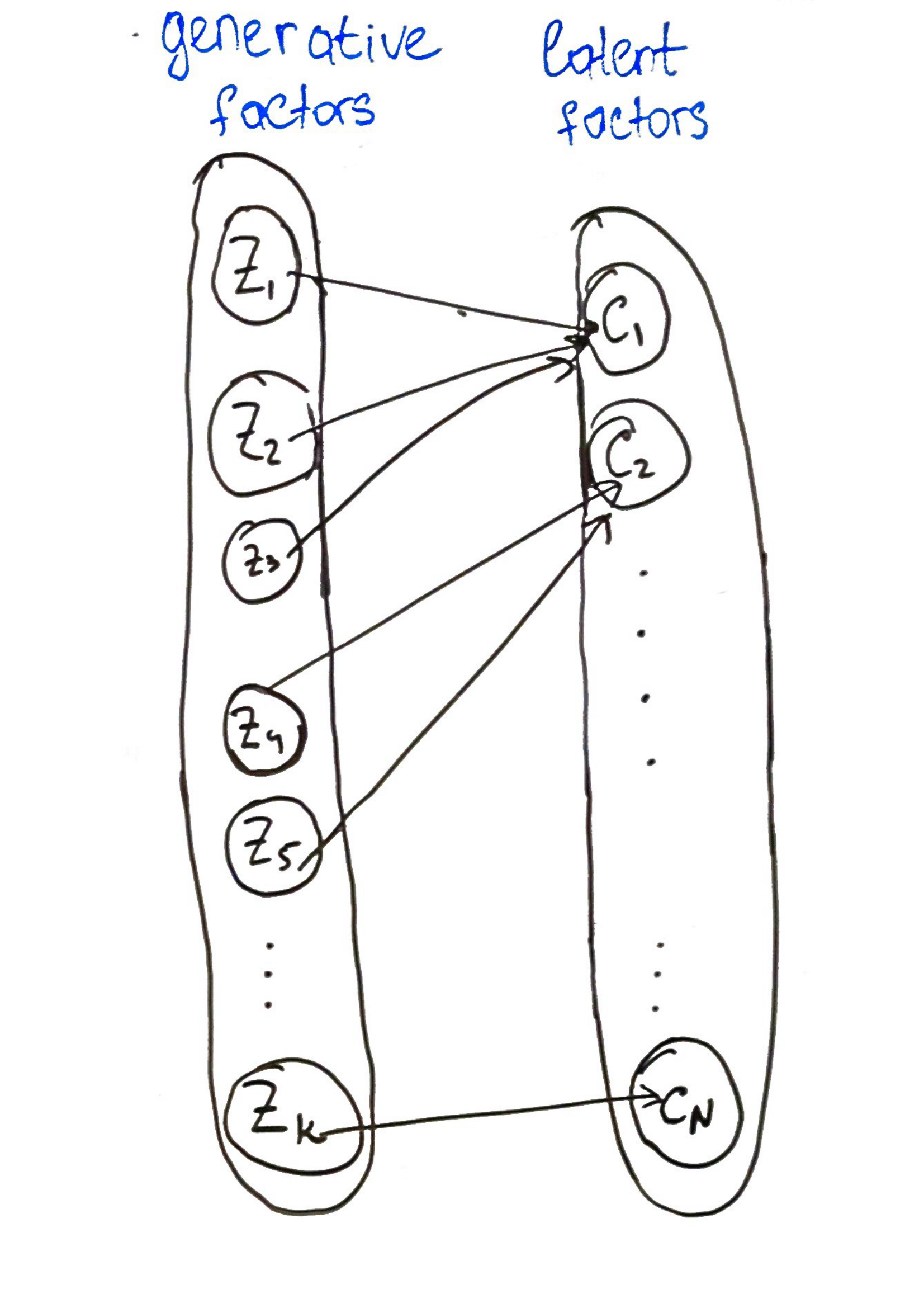}
\caption{Second characteristic of disentanglement}
\label{Second_caharact}
\end{subfigure}%
\caption{Different characteristics of disentanglement.}
\end{figure}

In this subsection, we analyze metrics that reflect the following characteristic of disentangled representations.
\begin{characteristic}
\label{dis_defin_first}
\rm
In a \emph{disentangled representation} a change in one latent dimension corresponds to a change in one generative factor while being relatively invariant to changes in other generative factors (see Fig.\ref{First_caharact}).
\end{characteristic}

\subsubsection{Definition of BetaVAE}
The algorithm that calculates \emph{BetaVAE}~\citep{higgins2016beta} consists of the following steps:
\begin{enumerate}[leftmargin=*,nosep]
\label{training_BetaVAE}
\item Choose a generative factor $z_r$.
\item Generate a batch of pairs of vectors for which the value of $z_r$ within the pair is equal, while other generative factors are chosen randomly:
\begin{equation*}
\begin{split}
(\bold{p}_1 = \langle z_{1,1},  \ldots, z_{1,K} \rangle, \bold{p}_2 = \langle z_{2,1}, \ldots, z_{2,K} \rangle),\\ z_{1,r} = z_{2,r}
\end{split}
\end{equation*}
\item Calculate the latent code of the generated pairs: $(\bold{c}_1 = f_e(g(\bold{p}_1)), \bold{c}_2 = f_e(g(\bold{p}_2)))$
\item Calculate the absolute value of the pairwise differences of these representations: 
\begin{equation*}
\bold{e}  =  \langle\,\abs{c_{1,1} - c_{2,1}}, \ldots, \abs{c_{1,N} - c_{2,N}}\rangle
\end{equation*}
\item The mean of these differences across the examples in the batch gives one training point for the linear regressor that predicts which generative factor was fixed.
\item BetaVAE is the accuracy of the linear regressor.
\end{enumerate}

\subsubsection{Definition of FactorVAE}
The idea behind FactorVAE~\citep{kim2018disentangling} is very similar to BetaVAE.
The main difference between them concerns how a batch of examples is generated to obtain a variation of latent variables when one generative factor is fixed. 
Another difference is the classifier that predicts which generative factor was fixed using the variation of latent variables.
\emph{FactorVAE} can be calculated by performing the following steps:
\begin{enumerate}[leftmargin=*,nosep]
\item Choose a generative factor $z_r$.
\item Generate a batch of vectors for which the value of $z_r$ within the batch is fixed, while other generative factors are chosen randomly.
\item Calculate latent codes of vectors from one batch.
\item Normalize each dimension in the latent representation by its empirical standard deviation over the full data.
\item Take the empirical variance in each dimension of these normalized representations.
\item The index of the dimension with the lowest variance and the target index $r$ provides one training point for the classifier.
\item FactorVAE is the accuracy of the classifier.
\end{enumerate}

\subsubsection{DCI: Disentanglement, Completeness and Informativeness}
\citet{eastwood2018framework} propose to use a metric of disentangled representations, which we call DCI, that is calculated as follows:
\begin{enumerate}[leftmargin=*,nosep]
\item First, the \emph{informativeness} between $c_i$ and $z_j$ is calculated.
To determine the informativeness between $c_i$ and $z_j$,~\citet{eastwood2018framework} suggest training $K$ regressors.
Each regressor $f_j$ predicts $z_j$ given $\mathbf{c}$ ($\hat{z_j} = f_j(\mathbf{c})$) and can provide an importance score $P_{i,j}$ for each $c_i$.
The normalized importance score obtained by regressor $f_j$ for variable $c_i$ is used as the informativeness between $c_i$ and $z_j$:
\begin{equation*}
I_{i,j} = \frac{P_{i,j}}{\sum_{k=0}^{k=K}P_{i,k}}.
\end{equation*}
\item For each latent variable its score of disentanglement is calculated as follows: 
\begin{equation*}
H_K(I_i)= 1 + \sum_{k=1}^{K}I_{i,k}\log_KI_{i,k}.
\end{equation*}
\item The weighted sum of the obtained scores of disentanglement for the latent variables is DCI: 
\begin{equation}
\label{weighted_sum}
\text{DCI}(\bold{c}, \bold{z}) =\ \sum_i\left(\rho_{i}\cdot H_K(I_i)\right),
\end{equation}
where $\rho_{i} ={\sum_{j}P_{i,j}}/{\sum_{ij}P_{i,j}}$.
\end{enumerate}

\subsubsection{Analysis of whether metrics satisfy the property~\ref{property_1}}

\begin{figure}[t]

\centering
\includegraphics[width=0.3\linewidth]{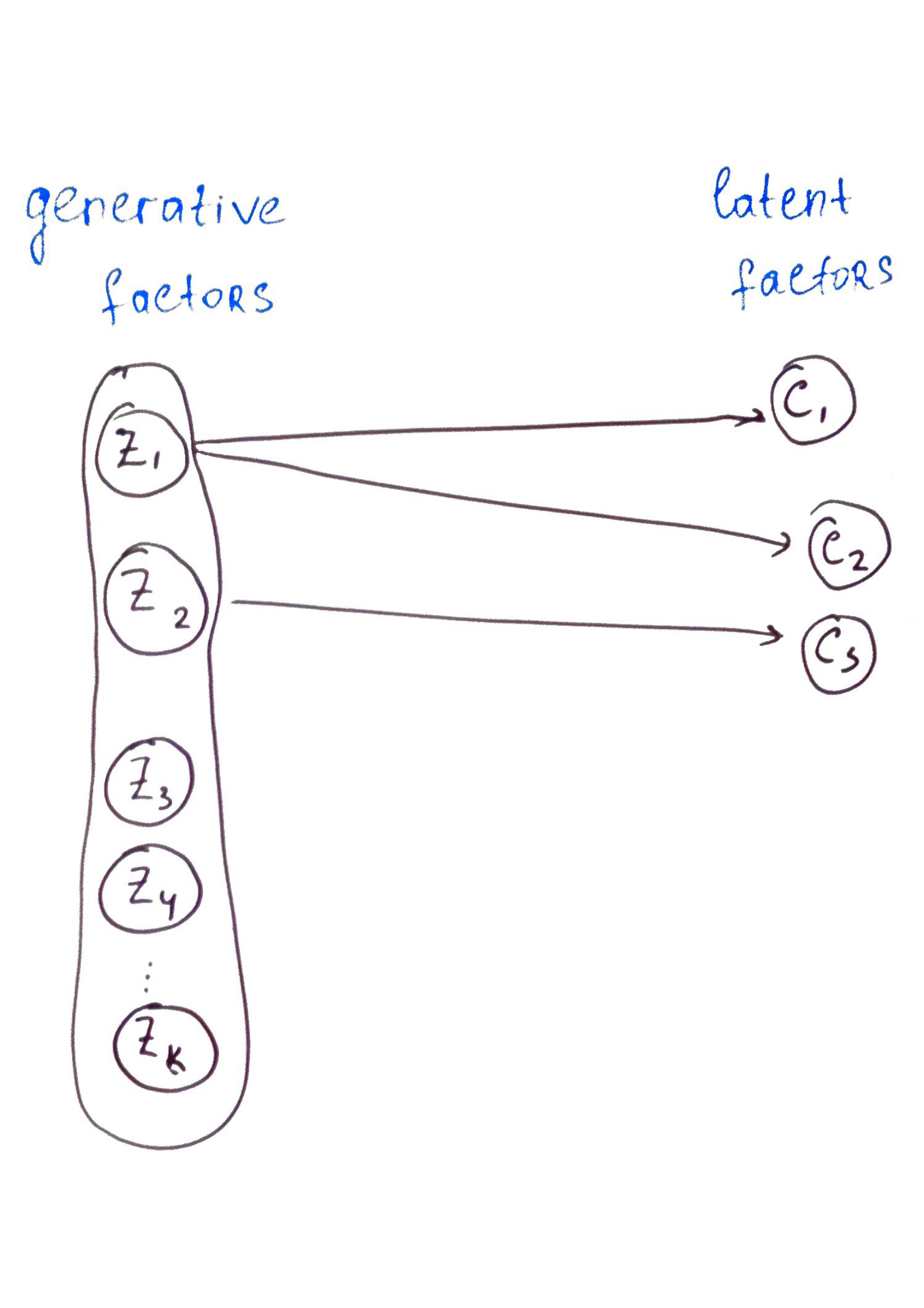}
\caption{Example of the representation, satisfying Characteristic~\ref{dis_defin_first}, but $BetaVAE = FactorVAE = frac{3}{K}$}
\label{ex1}
\end{figure}

\begin{fact}
\emph{BetaVAE} and \emph{FactorVAE} do not satisfy Property~\ref{property_1}.
\end{fact}
\begin{proof}
In a representation that satisfies Characteristic~\ref{dis_defin_first} there are could be several generative factors that are not captured by any latent factors. 
In this case BetaVAE and FactorVAE can not distinguish these generative factors.
\end{proof}

\begin{fact}
DCI does not satisfy Property~\ref{property_1}.
\end{fact}
\begin{proof}
We argue that using entropy as a score of disentanglement of one latent variable is not correct.
Indeed, a score of disentanglement of $c_i$ should be high when $c_i$ reflects one generative factor well, while it reflects other generative factors equally poorly.
However, since the distribution may be close to uniform for these generative factors, the entropy is large.
Let us provide an example that is built on this observation.
Suppose there are 11 generative factors, and 11 is the dimension of the latent representation.
Each latent factor $c_i$ captures primarily a generative factor $z_i$:
\begin{equation*}
I_{i,i} = 0.8,\ I_{i, k}=0.02,\ k \neq i.
\end{equation*}
Then, the DCI score is 0.6, so the DCI assigns a small score to a representation that satisfies Characteristic~\ref{dis_defin_first}.
\end{proof}

\subsubsection{Analysis of whether metrics satisfy the property~\ref{property_2}}

\begin{figure}

\begin{subfigure}[t]{0.48\linewidth}
\includegraphics[width=\linewidth]{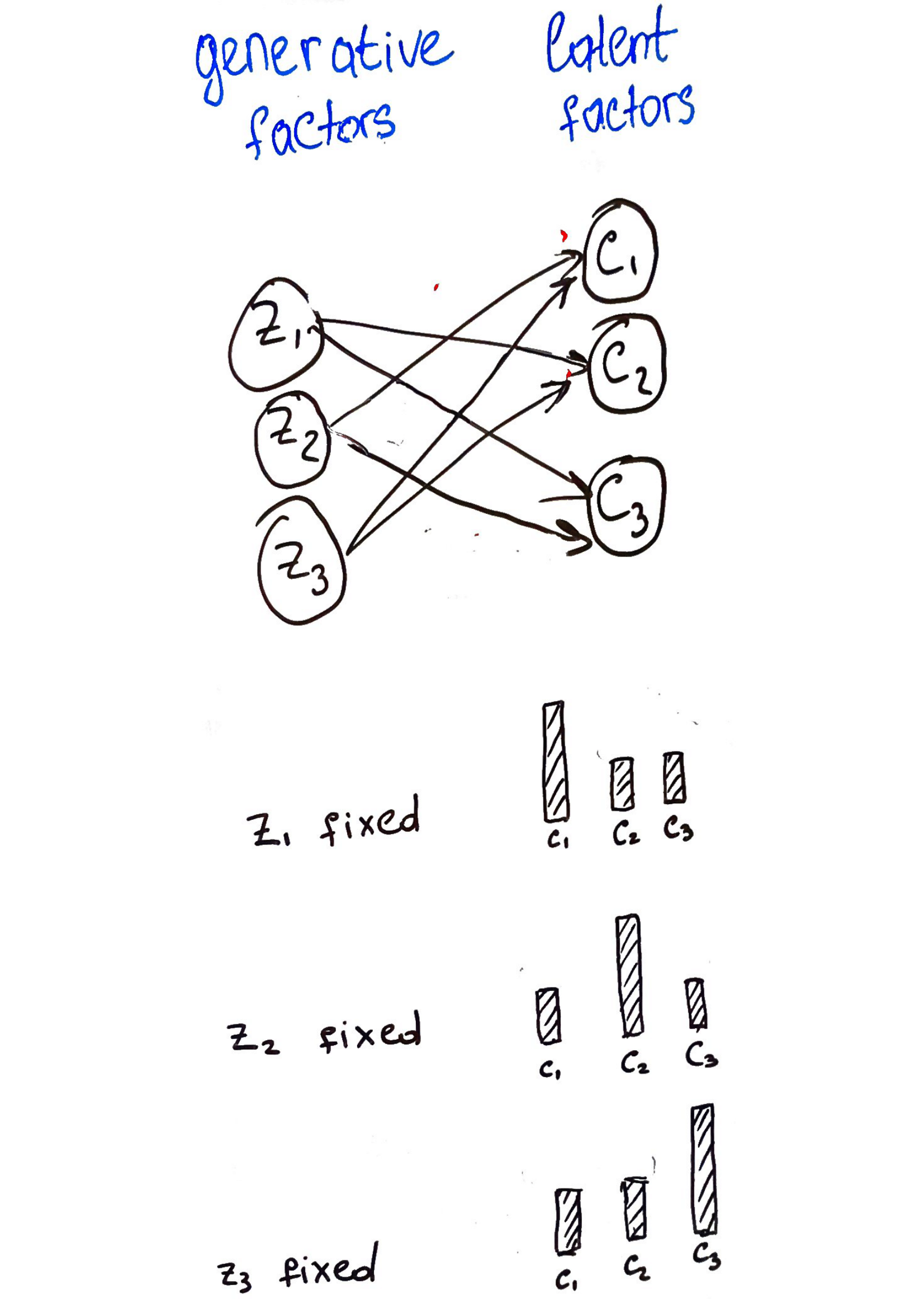}
\caption{Entangled representation with BetaVAE equal to 1}
\label{BetaVae}
\end{subfigure}%
\hfill
\begin{subfigure}[t]{0.48\linewidth}
\includegraphics[width=\linewidth]{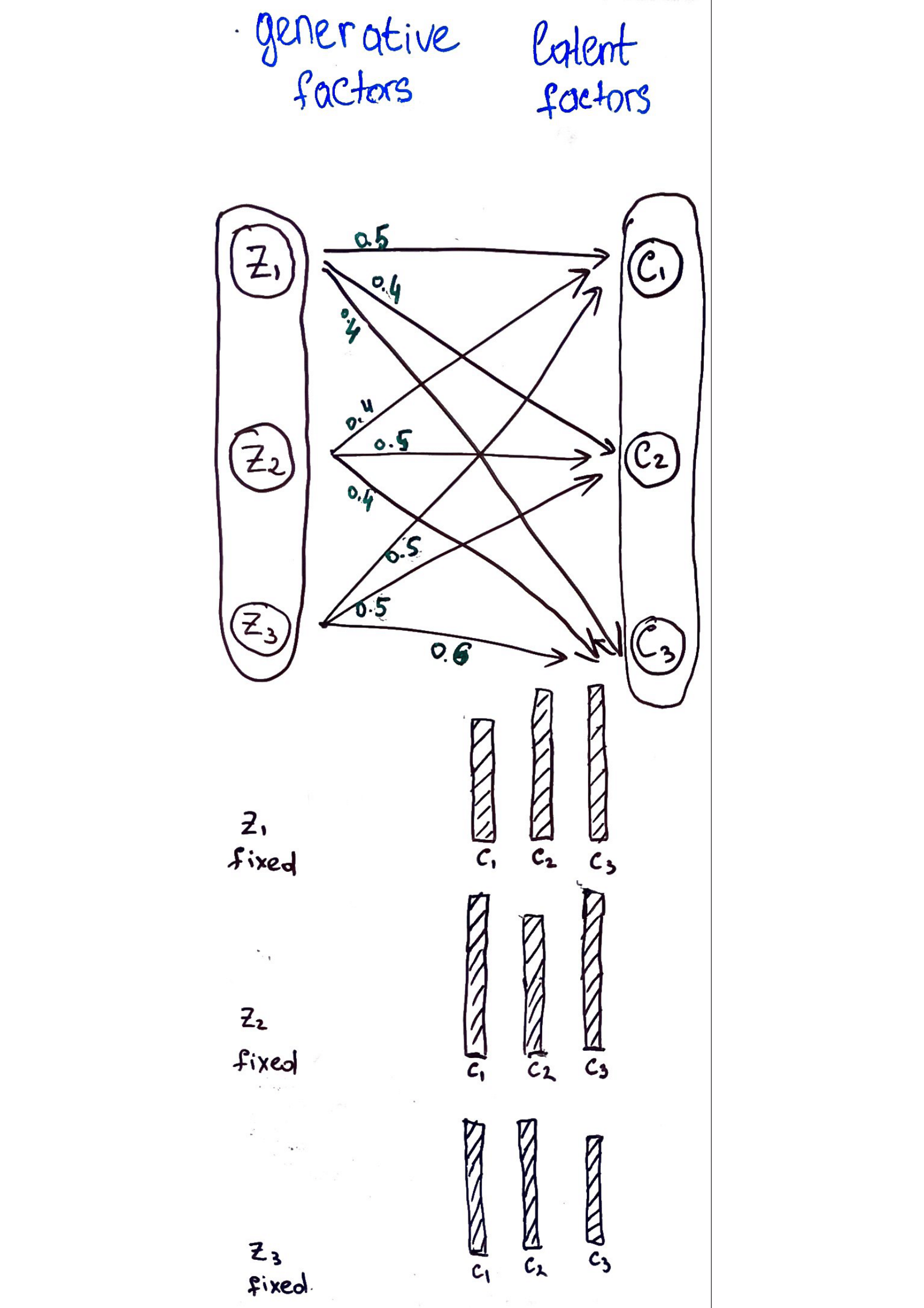}
\caption{Entangled representation with FactorVAE equal to 1}
\label{FactorVae}
\end{subfigure}%
\caption{Failures of BetaVAE and FactorVAE}
\end{figure}

\begin{fact}
\emph{BetaVAE} does not satisfy Property \ref{property_2}.
\end{fact}
\begin{proof}
As a proof, we give a counterexample (see Fig. \ref{BetaVae}).
Suppose there are 3 generative factors from a uniform distribution and the dimension of the latent representation is 3.
Assume that the latent variables are equal to the generative factors with the following probabilities:
\begin{equation*}
p_1 = (0.5,0.5, 0),\  
p_2 = (0, 0.5, 0.5),\  
p_3 = (0.5, 0, 0.5).
\end{equation*}

We generate 10,000 training points with a batch size of 128.
The accuracy of the linear classifier is equal to 0.9967 in this case, but the latent representation does not satisfy Characteristic~\ref{dis_defin_first}.
This shows that BetaVAE does not satisfy Property \ref{property_2}.
\end{proof}

\begin{fact}
\emph{FactorVAE} does not satisfy Property \ref{property_2}.
\end{fact}
\begin{proof}
Let us consider the following example (see Fig. \ref{FactorVae}).
Suppose there are 3 generative factors from a Gaussian distribution with $\mu = 0, \sigma = 1$, and each latent variable is a weighted sum of the generative factors:
\begin{equation*}
\begin{split}
c_1&=0.5 \cdot{}z_1+0.4\cdot{}z_2+0.5\cdot{}z_3\\ 
c_2&=0.4 \cdot{} z_1 + 0.5 \cdot{} z_2 +  0.5 \cdot{} z_3\\
c_3&=0.4 \cdot{} z_1 + 0.4 \cdot{} z_2 +  0.6 \cdot{} z_3.
\end{split}
\end{equation*}
We generate 10,000 training points with a batch size of 128.
The FactVAE disentanglement score is equal to 1 in this case, but the representation does not satisfy Characteristic~\ref{dis_defin_first}.
This shows that FactoVAE does not satisfy Property \ref{property_2}.
\end{proof}

\begin{fact}
DCI does not satisfy Property~\ref{property_2}.
\end{fact}
\begin{proof}
We give a counterexample, which is built on the fact that the weighted sum in Eq.~\ref{weighted_sum} can be large if only one latent variable is disentangled, while the other latent variables do not capture any information about generative factors.
Suppose there are 2 generative factors and the dimension of the latent representation is 2, and the matrix of informativeness is the following:
\begin{equation*}
P_{0, 0} = 1, P_{0,1} = 0, P_{1,1} = 0.09, P_{1,0} =0.01.
\end{equation*}
In this case, the DCI score is 0.957.
This counterexample shows that the DCI score can be close to 1 for the representation does not satisfy Characteristic~\ref{dis_defin_first}.
 \end{proof}

\subsection{SAP and MIG metrics}
In this subsection, we analyze metrics that reflect the following characteristic of disentangled representations.
\begin{characteristic}
\label{dis_defin_second}
\rm
In a \emph{disentangled representation} a change in a single generative factor leads to a change in a single factor in the learned representation (see Fig.~\ref{Second_caharact}).\footnote{\label{hidden_in_a_footnote}This property of representations is also called \emph{completeness}~\citep{eastwood2018framework}.}
\end{characteristic}

\subsubsection{SAP score: Separated Attribute Predictability}
\citet{kumar2017variational} provide a metric of disentanglement that is calculated as follows:
\begin{enumerate}[leftmargin=*,nosep]
\item Compute a \emph{matrix of informativeness} $I_{i,j}$, in which the $ij$-th entry is the linear regression or classification score of predicting the $j$-th generative factor using only the $i$-th variable in the latent representation.
\item For each column in the matrix of informativeness $I_{i,j}$, which corresponds to a generative factor, calculate the difference between the top two entries (corresponding to the top two most predictive latent factors). 
The average of these differences is the final score, which is called the SAP: 
\begin{equation*}
\text{SAP}(\bold{c}, \bold{z}) = \frac{1}{K} \sum_k \left(I_{i_k, k} - \max_{l\neq i_k} I_{l,k}\right),
\end{equation*}
where $i_k = \argmax_i I_{i,k}$. 
\end{enumerate}

\subsubsection{MIG: Mutual Information Gap}
\citet{chen2018isolating} propose a disentanglement metric, Mutual Information Gap (MIG), that uses mutual information between the $j$-th generative factor and the $i$-th latent variable as a notion of informativeness between them.
The \emph{mutual information} between two variables $c$ and $z$ is defined as
\begin{equation*}
I(c;z) = H(z) - H(z|c),
\end{equation*}
where $H(z)$ is the entropy of the variable $z$.
Mutual information measures how much knowing one variable reduces uncertainty about the other.
A useful property of mutual information is that it is always non-negative $I(c;z) > 0$. 
Moreover, $I(c;z)$ is equal to 0 if and only if $c$ and $z$ are independent.
Also, mutual information achieves its maximum if there exists an invertible relationship between $c$ and $z$. 
The following algorithm calculates the MIG score:
\begin{enumerate}[leftmargin=*,nosep]
\item Compute a \emph{matrix of informativeness} $I_{i,j}$, in which the $ij$-th entry is the mutual information between the $j$-th generative factor and the $i$-th latent variable.
\item For each column of the score matrix $I_{i,j}$, which corresponds to a generative factor, calculate the difference between the top two entries, and normalize it by dividing by the entropy of the corresponding generative factor.
The average of these normalized differences is the MIG score:
\begin{equation*}
\text{MIG}(\bold{c}, \bold{z}) = \frac{1}{K} \sum_k \frac{I_{i_k, k} - \max_{l\neq i_k} I_{l,k}}{H(z_k)},
\end{equation*}
where $i_k = \argmax_i I_{i,k}$.
\end{enumerate}

\subsubsection{Analysis of whether metrics satisfy the property~\ref{property_1}}
\begin{fact}
\label{SAP_score_ex}
SAP does not satisfy Property \ref{property_1}.
\end{fact}
\begin{proof}
We claim that it is incorrect to use the $R^2$ score of linear regression as informativeness between latent variables and generative factors.
Indeed, a linear regression cannot capture non-linear dependencies.
Thus, informativeness, which is calculated using the $R^2$ score of a linear regression, may be low if each generative factor is a non-linear function of some latent variable.
Let us give an example that is built on this observation.
Suppose there are 2 generative factors from the uniform distribution $U([-1,1])$ and the dimension of the latent representation is 2.
Let us assume the latent variables are obtained from the generative factors according to the following equations:
\begin{equation*}
c_1 = z_1^{15},\ 
c_2 = z_2^{15}.
\end{equation*}
For this representation, we generate 10,000 examples and obtain the SAP score equal to 0.32.
It proves that SAP can assign a low score to a representation that satisfies Characteristic~\ref{dis_defin_second}.
\end{proof}

\begin{fact}
MIG satisfies Property~\ref{property_1}.
\end{fact}
\begin{proof}
Indeed, in a disentangled representation each generative factor is primarily captured in only one latent dimension.
This means that for each generative factor $z_j$, there is exactly one latent factor $c_{i_j}$ for which $z_j$ is a function of $c_{i_j}$: $z_j \sim f(c_{i_j})$.
Therefore,
\begin{equation*}
I_{i_ j, j} = H(z_j) - H(z_j|c _{i,j}) \sim H(z_j),
\end{equation*}
whereas for other latent variables $I_{k,j} = I(c_k, z_j) \sim 0$.
Consequently, according to MIG, the score of disentanglement of each generation factor is close to 1:
\begin{equation}
\frac{I_{i_j, j} -\max_{k \neq i_j} I_{k,j}}{H(z_j)} \sim 1.
\end{equation}
Therefore, the average of these scores is also close to 1.
This shows that MIG always assigns a high score to a representation that satisfies Characteristic~\ref{dis_defin_second}.
\end{proof}

\subsubsection{Analysis of whether metrics satisfy the property~\ref{property_2}}

\begin{fact}
SAP does not satisfy Property~\ref{property_2}.
\end{fact}
\begin{proof}
A high SAP score indicates that the majority of generative factors is captured linearly in only one latent dimension.
However, the SAP metric does not penalize the existence of several latent factors that capture the same generative factor non-linearly.
Let us consider the following example.
Suppose there are 2 generative factors from the uniform distribution $U([-1,1])$, and the dimension of the latent representation is 3.
Let us assume that the latent factors are obtained from the generative factors according to the following equations:
\begin{equation*}
c_1 = z_1,\ 
c_2 = z_1^{25} + z_2^{25},\
c_3 = z_2.
\end{equation*}
For this latent representation, a change in each generative factor leads to a change in several latent factors, but the SAP score is equal to 0.98.
This shows that the SAP score can be close to 1 for a latent representation that does not satisfy Characteristic~\ref{dis_defin_first}.
\end{proof}
\begin{fact}
MIG satisfies Property~\ref{property_2}.
\end{fact}
\begin{proof}
A high MIG score indicates that the majority of generative factors is captured in only one latent dimension. 
Consequently, a change in one of the generative factors entails a change primarily in only one latent dimension.
\end{proof}

A summary of the results of our analysis is given in Table~\ref{result-table}.

\begin{table}[h]
  \caption{Summary of facts about proposed metrics of disentangled representations.}
  \label{result-table}
  \centering
  \begin{tabular}{lcc}
    \toprule
	Metric & Satisfies Property~\ref{property_1} & Satisfies Property~\ref{property_2} \\
    \midrule
BetaVAE&	No&	 No\\
FactorVAE& No&	No\\
DCI&	 No& No\\
SAP&	No&	 No\\
MIG&	Yes&	 Yes\\
    \bottomrule
  \end{tabular}
\end{table}

\subsection{Difference between Characteristics~\ref{dis_defin_first} and~\ref{dis_defin_second}}
The Characteristics~\ref{dis_defin_first} and~\ref{dis_defin_second} of a disentangled representation have important differences.
Indeed, a representation in which several latent factors capture one common generative factor satisfies a Characteristic~\ref{dis_defin_first}, but not a Characteristic~\ref{dis_defin_second}.
On the other hand, a representation in which a latent variable captures multiple generative factors while there are no other latent variables that capture these generative factors does not satisfy Characteristic~\ref{dis_defin_first}, but satisfies Characteristic.~\ref{dis_defin_second}.

Consider, for example, the following latent representation of dimension 4 of the dataset containing rectangles of different shapes shown in Fig.~\ref{ex}:
$$
z_1 = x,\ 
z_2 = x^2,\ 
z_3 = y,\ 
z_4 = y^3
$$
where $x$ is the length of a rectangle, while $y$ is the width of a rectangle.
It satisfies Characteristic~\ref{dis_defin_first}, but not a Characteristic~\ref{dis_defin_second}.
Conversely, any one-dimensional latent representation of the same dataset would satisfy Characteristic~\ref{dis_defin_second}, but not necessarily Characteristic~\ref{dis_defin_first}.

\section{A New Metric of Disentanglement, \OurMeasure} 
The previous metrics were designed to reflect only one out of two characteristics of disentangled representations. 
We believe that a metric should reflect both of them: Characteristics~\ref{dis_defin_first} and~\ref{dis_defin_second}.
Moreover, following \citep{eastwood2018framework} we think that the metric should also reflect the \emph{informativeness} of a representation.

Formally, this means that we believe that a \emph{disentangled representation} satisfies the following characteristic.

\begin{characteristic}
\label{dis_defin_third}
\rm
In a disentangled representation, we can choose a subset of latent variables:
$c' = \left\{ c_{i_1}, \dots, c_{i_K}\right\}$,     
that satisfy Characteristic~\ref{dis_defin_first} and Characteristic~\ref{dis_defin_second}.
Moreover, a disentangled representation should contain nearly all information about generative factors, i.e., it should have a high degree of \emph{informativeness}~\citep{eastwood2018framework}.
\end{characteristic}

With this in mind, we propose a new metric of disentanglement of representation, called \OurMeasure.

\subsection{Definition of \OurMeasure}
The algorithm for calculating \OurMeasure{} is completely different from previously proposed metrics in the following sense.
Differently from \OurMeasure{}, all conventional metrics do not make one-to-one correspondence between generative factors and latent factors.
On the other hand, the main part of algorithm calculating \OurMeasure{} is finding for each generative factor the correspondent latent factor.
At the first stage, for each generative factor, \OurMeasure{} selects a set of candidates from latent variables that may correspond to the generative factor (see Fig. \ref{first_stage}).
In the second stage, for each generative factor, \OurMeasure{} selects the best latent variable from the set of corresponding candidates (see Fig. \ref{second_stage}).
\begin{figure}

\begin{subfigure}[t]{0.48\linewidth}
\includegraphics[width=\linewidth]{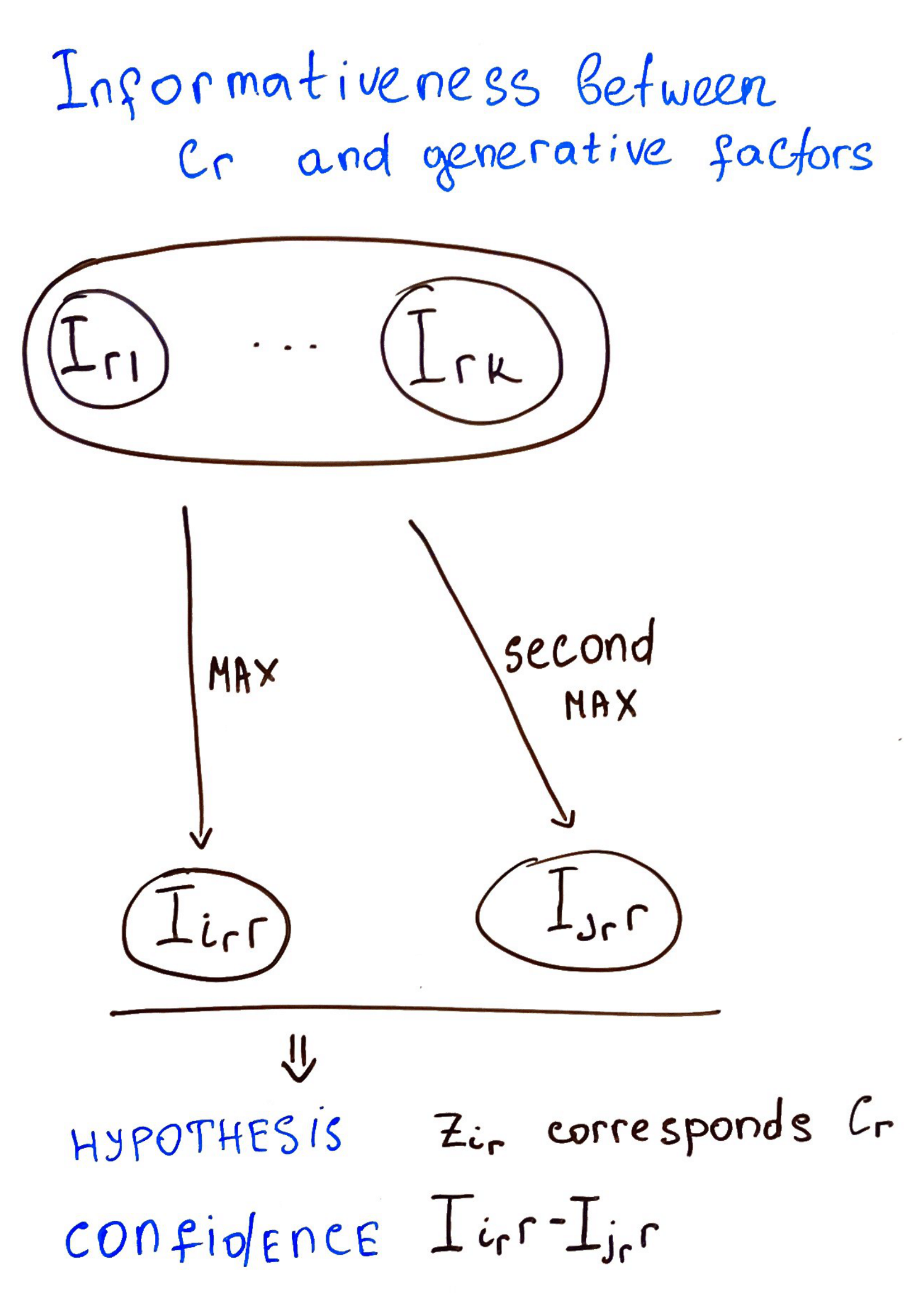}
\caption{First stage of \OurMeasure{}: making hypothesis}
\label{first_stage}
\end{subfigure}%
\hfill
\begin{subfigure}[t]{0.48\linewidth}
\includegraphics[width=\linewidth]{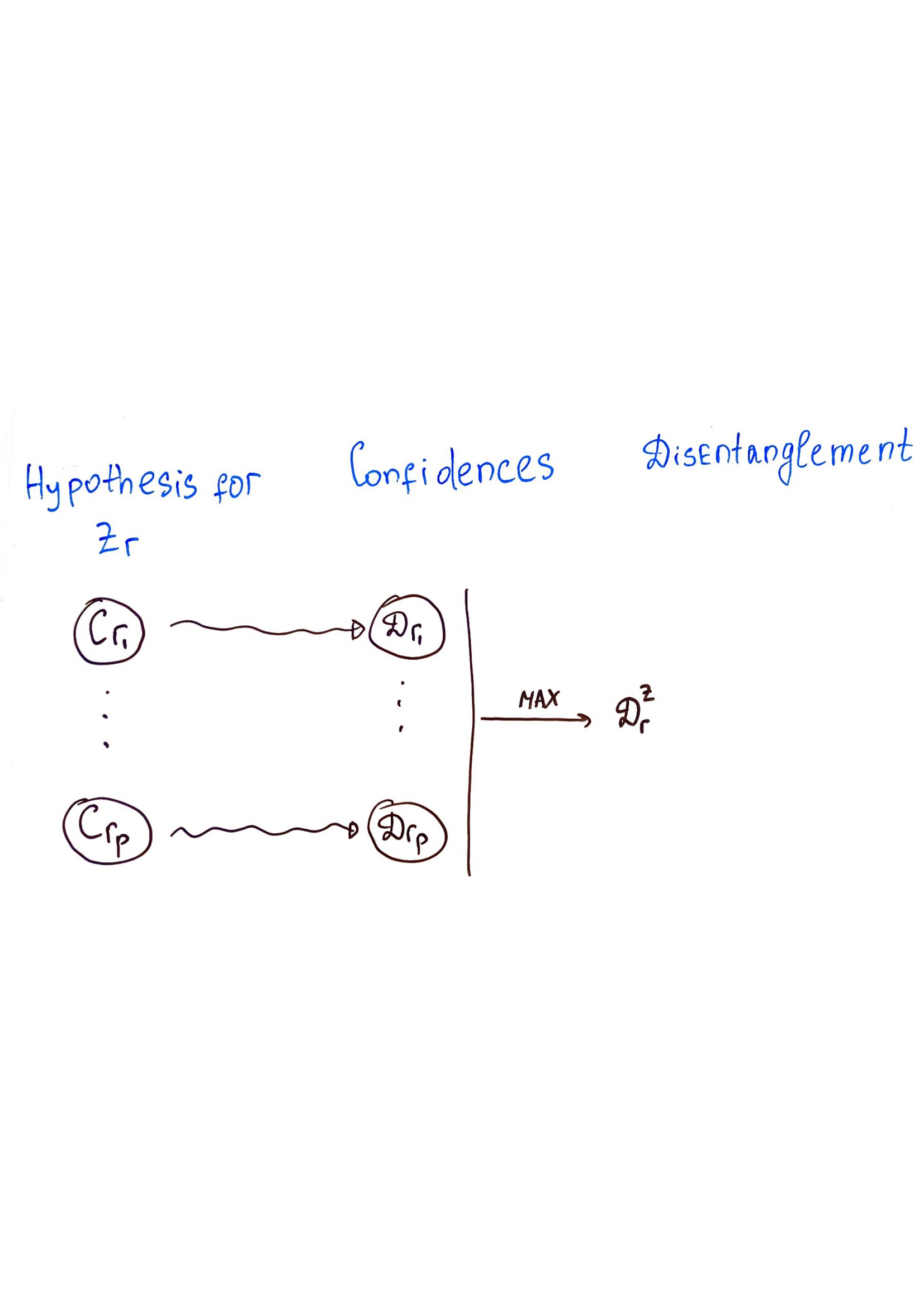}
\caption{First stage of \OurMeasure{}: choosing the best hypothesis}
\label{second_stage}
\end{subfigure}%
\caption{\OurMeasure}
\end{figure}

More formally, first, we create a matrix of informativeness $I_{i,j}$, in which the $ij$-th entry is the mutual information between the $j$-th generative factor and the $i$-th variable in the latent representation.

\begin{enumerate}[leftmargin=*,nosep]
\item For each latent variable $c_i$, find the generative factor $z_{j_i}$ that it reflects the most: $j_i = \argmax_j I_{i,j}$. 
\item Calculate the disentanglement for each latent variable: $D_{i} = I_{i, j_i} - \max_{k\neq j_i} I_{i,k}$. 
\item For each generative factor $z_j$, find the most disentangled latent factor $c_{k_j}$, that reflects $z_j$: $k_j = \argmax_{l\in \mathbb{I}_j} D_l$, where $\mathbb{I}_j = \{i: z_{j_i} = z_j\}$. 
\item For each generative factor $z_j$, calculate the disentanglement score $D^z_j$, which is equal to $D_{k_j}$ if there is at least one latent factor, that captures $z_j$, otherwise, it is 0. 
\item Finally, the \emph{disentanglement score} of a latent representation according to \OurMeasure{} is the normalized sum of $D^z_j$:
\begin{equation*}
\text{\OurMeasure}(\bold{c}, \bold{z}) = \frac{\sum_{j=1}^K D^z_j}{\sum_{j=1}^K H(z_j)},
\end{equation*}
 where $H(z_j)$ is the entropy of $z_j$.
\end{enumerate}

\subsection{Analysis of whether \OurMeasure{} satisfies Property~\ref{property_1}}
\begin{fact}
\OurMeasure{} satisfies  Property~\ref{property_1}.
\end{fact}
\begin{proof}
Indeed, in a representation that satisfies Characteristic~\ref{dis_defin_third}, there is a subset $c'$ of latent variables, in which each latent variable is sensitive to changes in one generative factor only. 
Moreover, for each generative factor $z_j$ there is only one latent variable $c_{i_j} \in c'$ that captures the changes in $z_j$.
Consequently, $c_{i_j}$ is a function of $z_j$: $c_{i_j} = f_j(z_j)$, while the other latent factors are invariant to changes in $z_j$.
This means that, $D_{i_j} = I_{i_j, j} - \max_{k\neq j} I_{i_j,k} = I_{i_j, j}$.
Also, the disentangled representation should have a high degree of \emph{informativeness}.
Consequently, the latent variables in $c'$ should capture all the information contained in $z_j$.
But only $c_{j, i}$ contains some information about $z_j$.
Therefore, $ I_{i_j, j} = H(z_j)$, and $D^z_j = H(z_j)$.
Consequently, \OurMeasure{} is equal to 1 in this case.
\end{proof}

\subsection{Analysis of whether \OurMeasure{}  satisfies Property~\ref{property_2}}
\begin{fact}
\OurMeasure{} satisfies  Property~\ref{property_2}.
\end{fact}
\begin{proof}
When a representation does not satisfy Characteristic~\ref{dis_defin_third} for the majority of informative generative factors $z'$, we cannot find a factor in the latent representation that reflects only this factor.
There are 2 cases for the generative factors from $z'$.
In the first case, there is no latent factor that captures the generative factor $z_j \in z'$.
In that case, $D^z_j$ is equal to 0.
The second case is characterized by the fact that there is a latent factor that captures a generative factor, but this latent factor also captures other generative factors.
In that case the disentangled score of this latent factor $D_{i_j}$ is small, and consequently, $D^z_j$ is small.
\end{proof}
%

\section{Differencese between metrics}
In this section, we explore the differences between metrics in more depth.
We give very simple example of informativeness between generative and latent factors and show that different metrics give different scores and do not correlate with each other.
In the example we assume that the entropy of any generative factor is equal to 1.
Also we do not discuss results of SAP score because it highly dependent on linearity of generative process, see Section \ref{SAP_score_ex}.

\begin{figure}[t]

\centering
\includegraphics[width=0.6\linewidth]{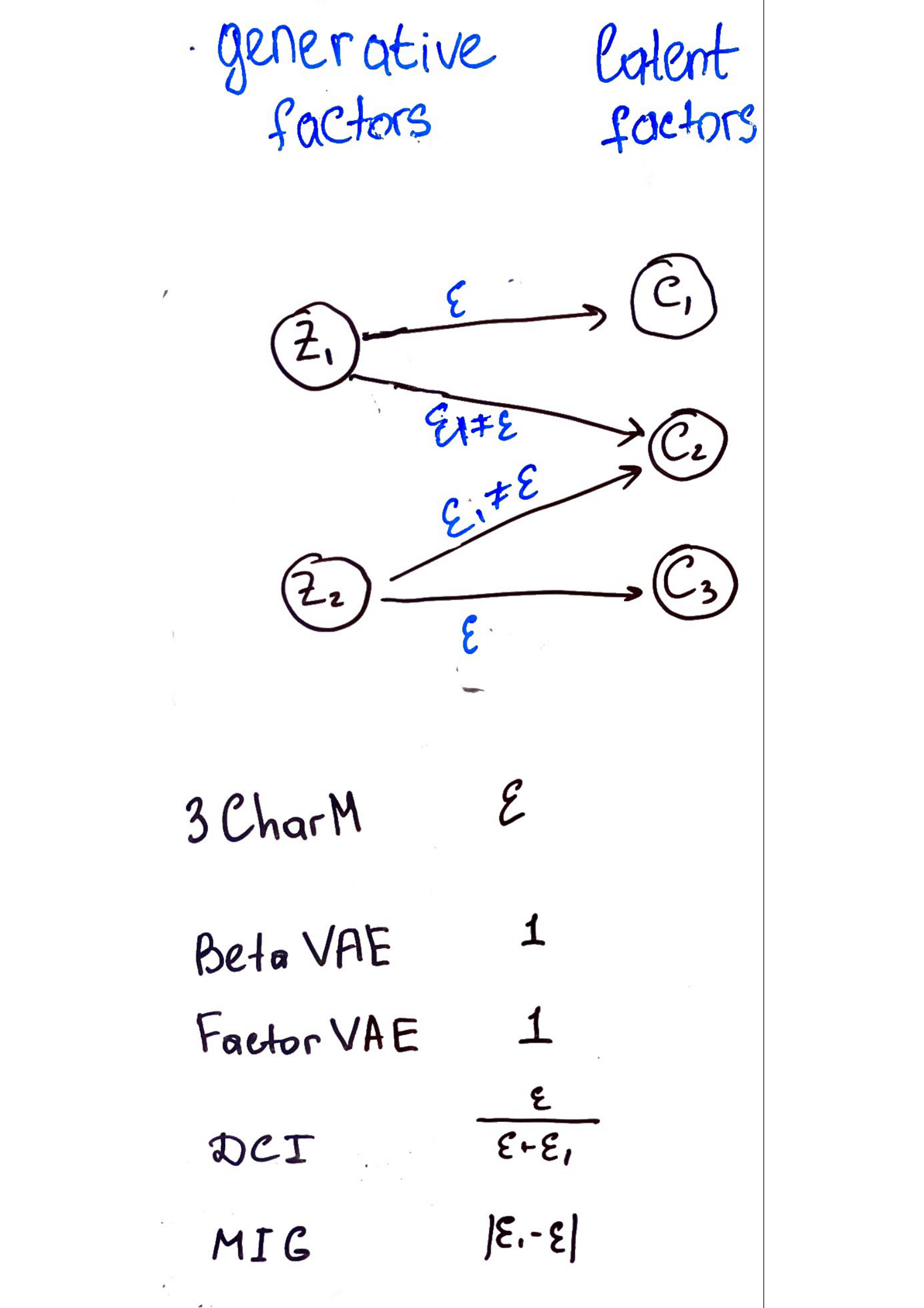}
\caption{Example of informativeness between generative and latent factors and correspondent metrics' scores}
\label{ex1}
\end{figure}

\subsection{Simple example that illustrates the differences between metrics}
In this example we show the representation, that has 2 parameters.
For this representation BetaVAE and FactorVAE \emph{always} give perfect score \emph{independently} of parameters.
The value given by \OurMeasure{} depends only on 1 parameter, while the values of MIG and DCI depend on 2 parameters.
Moreover \OurMeasure{}, MIG and DCI can give completely different results.

Indeed, let us consider the data with 2 generative factors and 3 latent factors, given in the Fig. \ref{ex1}.
The matrix of informativeness is equal to the following 
$$I_{1,1} = I_{2,3} = \epsilon,\ I_{1,2} = I_{2,2} = \epsilon_1, \epsilon\neq\epsilon_1.$$
Then, as simple calculations show 
$$\OurMeasure = \epsilon,\ MIG = |\epsilon - \epsilon_1|,\ DCI = \frac{\epsilon}{\epsilon + \epsilon_1}.$$
This shows that the values of metrics can have different values as shown in the Table~\ref{metrics-values-table}.

\begin{table}[h]
  \caption{Values of metrics for representation shown in Fig. \ref{ex1} depending of values $\epsilon, \epsilon_1$}
  \label{metrics-values-table}
  \centering
  \resizebox{\linewidth}{!}{
  \begin{tabular}{lccccc}
    \toprule
	Values of parameters& \OurMeasure{}& BetaVAE& FactorVAE&  DCI& MIG\\
    \midrule
$\epsilon\sim 0, \epsilon_1\sim ~1$&	$\sim 0$& 1& 1& $\sim 0$& $\sim ~1$\\
$\epsilon\sim 0, \frac{\epsilon_1}{\epsilon}\sim ~0$&	$\sim 0$& 1& 1& $\sim ~1$& $\sim 0$\\
$\epsilon\sim 0, \epsilon_1\sim 0, \frac{\epsilon_1}{\epsilon}\gg 1$&	 $\sim 0$& 1& 1& $\sim 0$& $\sim 0$\\
$\epsilon\sim 1, \epsilon_1\sim 1$&	$\sim ~1$& 1& 1& $\sim 0$& $\sim ~0.5$\\
$\epsilon\sim 1, \epsilon_1\sim 0$&	$\sim ~1$& 1& 1& $\sim ~1$& $\sim ~1$\\

    \bottomrule
  \end{tabular}}
\end{table}

\subsection{Spearman rank correlation between metrics}
Following \citep{locatello2018challenging}, we explore how the metrics agree.
We provide tables with correlation scores between the metrics in Fig.~\ref{rank_correllation}.
We expand the tables given in~\citep{locatello2018challenging}, which show the correlation of Spearman ranks between different metrics, by adding \OurMeasure{}.
We show the results for two datasets: dSprites~\citep{higgins2016beta} and Cars3D~\citep{reed2015deep}, in Fig.~\ref{rank_correllation}.

\begin{figure}[h]
\centering
\begin{subfigure}[t]{0.45\linewidth}
\includegraphics[clip,width=\linewidth]{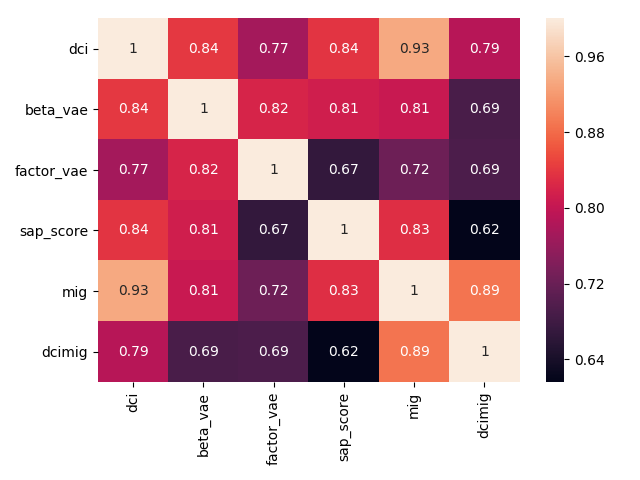}
\caption{Rank correlation of different metrics on the dSprites dataset.}
\end{subfigure}
\hfill
\begin{subfigure}[t]{0.45\linewidth}
\includegraphics[clip, width=\linewidth]{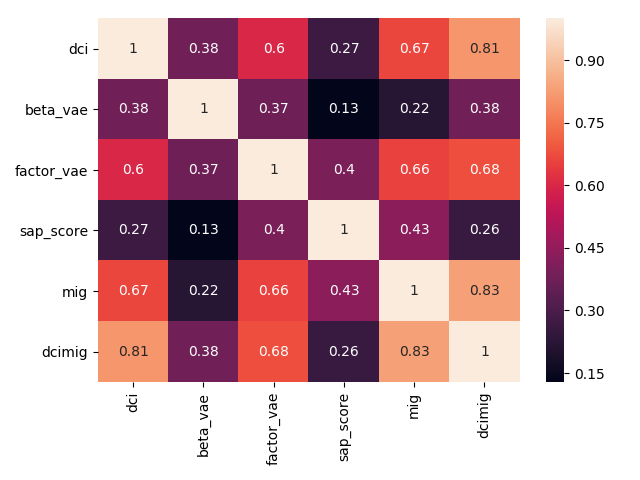}
\caption{Rank correlation of different metrics on the Cars3D dataset.}
\end{subfigure}%
\caption{Rank correlation of different metrics on two datasets. Overall, all metrics are strongly correlated.}
\label{rank_correllation}
\end{figure}

\textbf{Results.}
These datasets are artificial datasets, on which we observe that all metrics are strongly correlated.

\section{Related Work}
This paper is relevant to two research directions: the formulation of a notion of disentangled representation and the analysis of differences between proposed metrics of disentangled representations.

A definition of disentangled representation is presented by~\citet{higgins2018towards}, who propose to call a representation disentangled if it is consistent with transformations that characterized the dataset.
In particular, \citet{higgins2018towards} suggested that transformations that change only some properties of elements in the dataset, while leaving other properties unchanged, give the structure of a dataset.
Desirable properties of a disentanglement metric are formulated by~\citet{eastwood2018framework}; they are \emph{disentanglement}, \emph{completeness}, and \emph{informativeness}.
\citet{eastwood2018framework} claim that a good representation should satisfy all of these properties, namely (1)~if a representation is  good, then change in one latent factor should lead to change in one generative factor, (2)~a change in one generative factor should lead to a change in one latent factor, and (3)~a latent representation should contain all information about the generative factors.
Therefore, \citet{eastwood2018framework} propose three metrics to satisfy each of the properties listed. However, the proposed metrics were not analyzed --- a gap that we fill.

Several papers analyze the differences between metrics of disentanglement through experimental studies~\citep{locatello2018challenging, chen2018isolating}.
For example, \citet{locatello2018challenging} train 12,000 models that cover the most prominent methods and evaluate these models using existing metrics of disentanglement.
The study shows that the metrics are correlated, but the degree of correlation depends on the dataset.
It is important to note that their experimental results are consistent with our theoretical findings: the BetaVAE~\citep{higgins2016beta} and FactorVAE~\cite{kim2018disentangling} metrics are strongly correlated with each other; and the SAP~\citep{kumar2017variational}, MIG~\cite{chen2018isolating}, DCI~\citep{eastwood2018framework} scores are also strongly correlated.
\citet{locatello2018challenging} take an important step towards the evaluation of methods to create disentangled representations, however, the properties of the metrics are not analyzed theoretically.
\citet{chen2018isolating} take a step in this direction, but only analyze the BetaVAE, FactorVAE and MIG metrics. 
\citet{chen2018isolating} compare metrics by analyzing robustness to the choice of the hyperparameters during experiments.
The experimental findings are quite similar to ours: BetaVAE is a very optimistic metric and assigns high scores to entangled representations.

To summarize, the key distinctions of our work compared to previous efforts are: (1)~a broad coverage, in-depth analysis of previously proposed metrics of disentanglement, and (2) a proposal of a single metric of disentanglement that reflects all properties of previously proposed ones and has theoretical guarantees.

\section{Conclusion}
In recent years, several models have been developed to obtain disentangled representations~\citep{yu2017semantic, hu2017disentangling, denton2017unsupervised, kim2018disentangling}.
Currently, there are five metrics that are commonly used to evaluate the models: \emph{BetaVAE}~\citep{higgins2016beta}, FactorVAE~\citep{kim2018disentangling}, DCI~\citep{eastwood2018framework}, SAP~\citep{kumar2017variational} and MIG~\citep{chen2018isolating}.
Interestingly, all of these metrics are based upon the definition of disentangled representation proposed in~\citep{bengio2013representation}.
However, three of the metrics were designed to reflect Characteristic~\ref{dis_defin_first} of disentangled representations, while two were designed to reflect Characteristic~\ref{dis_defin_second}.
The primary goal of this paper has been to provide an analysis of the existing metrics of disentangled representations.
We theoretically analyze how well the proposed metrics reflect the characteristics of disentangled representations that they are intended to reflect.
In particular, we analyzed each of the existing metrics of disentanglement by two properties: whether a metric is close to $1$ when a representation satisfies the characteristic that the  metric reflects and whether the metric is close to $0$ when a representation does not satisfy the characteristic.
Surprisingly, we found that most of the existing metrics does not satisfy these basic properties.

The importance of developing a reliable metric of disentanglement has been clearly stated by~\citet{kim2018disentangling, abdi2019preliminary}.
A key contribution of this paper is a new metric of disentangled representation, called \OurMeasure{}.
First, we formalize the desired characteristics, which, in our opinion, should reflect the metrics, and then prove that \OurMeasure{} reflects them properly.

\bibliographystyle{icml2020}
\newpage
\vskip 0.2in
\bibliography{icml2020}


\end{document}